\documentclass[12pt,letterpaper,onecolumn]{article}
\usepackage{url}

\usepackage{color}
\usepackage{amsmath}
\usepackage{amsfonts}
\usepackage{amssymb}
\usepackage{cite}
\usepackage[dvips]{graphicx}
\usepackage{amsthm}
\usepackage{algorithm}
\usepackage{algorithmic}
\usepackage[retainorgcmds]{IEEEtrantools}
\usepackage{caption}
\usepackage{setspace}
\theoremstyle{definition} \newtheorem{theorem}{Theorem}
\theoremstyle{definition} \newtheorem{lemma}{Lemma}
\theoremstyle{definition} \newtheorem{corollary}{Corollary}
\theoremstyle{remark} \newtheorem{remark}{Remark}

\usepackage[margin=1in]{geometry}

\newcommand{\bydef}{\stackrel{\bigtriangleup}{=}}
\newcommand{\orans}{} % {\color{red}}
\newcommand{\orane}{} %{\color{black}}

\title{ How to Allocate Resources For Features Acquisition?}
\author{
Oran Richman 
Department of Electrical Engineering
Technion
Haifa, Israel\\ 
\texttt{roran@tx.technion.ac.il} \\
Shie Mannor 
Department of Electrical Engineering
Technion
Haifa, Israel \\
\texttt{shie@ee.technion.ac.il} 
}

\begin{document}

\maketitle

\begin{abstract}
We study classification problems where features are corrupted by noise and where
the magnitude of the noise in each feature is influenced by the resources allocated to its acquisition. This is the case, for example, when multiple sensors share a common resource (power, bandwidth, attention, etc.). We develop a method for computing the optimal resource allocation for a variety of scenarios and derive theoretical bounds concerning the benefit that may arise by non-uniform allocation. We further demonstrate the effectiveness of the developed method  in simulations. 
\end{abstract}

\section{Introduction}

Most machine learning settings take feature vectors as input. These features are often acquired using some process resulting in less than optimal data quality.   In many situations, the data quality depends on the resources allocated for the data acquisition process. Examples of possible resources are   sample rate, total sample time, CPU allocated to some costly pre-processing, and transmitted power.

 Several approaches have been proposed in order to  deal with some uncertainty in learning schemas (for example\cite{xu2009robustness,shivaswamy2006second}).    
In many cases, however, one does not   merely deal with existing uncertainty, but can sometimes ``shape'' the uncertainty to meet one's needs. This is often the case when several sensors share a common resource. For example, mobile applications use sensors that share power, CPU and bandwidth. Each of those resources can be divided between sensors according to the designer wish. Another example is the design of a system with fixed budget (money wise), each type of sensor incorporated can have a variety of qualities (with a price tag to match). Which sensor is ``worth'' investing in? 

In this work we explore the following problem: Several sensors  that share a common resource  acquire inputs that will be used for classification. What is the best way to divide  the resource between the sensors? The resources allocated for each sensor affect the quality of the data it collects. We wish to maximize classification performance by correctly allocating the available resources.   We emphasize that different resource allocation schemes may result in different optimal classifiers. This coupling increases the complexity of the problem. %We will explore multiple cases, both for the disturbance characteristics (stochastic or adversarial), % the stage in which the disturbance is introduced ( in the learning phase or the inference phase) and different loss functions. 

We present a framework for ``uncertainty management": This framework formulates the presented problem as an  optimization problem. The direct formulation, however, is not easily solvable so
we derive an equivalent solvable problems for various scenarios. 
We further bound the benefit that may arise from optimally allocating the resources. 
%both for the case when the disturbance is adversarial and to the case where the disturbance is stochastic. 
Based on the results presented we devise an algorithm for deriving the optimal resource allocation and present some simulation results that show the potential benefits.

An application domain of such an approach is that of sensor management (see \cite{hero2011sensor}), where mostly state-estimation problems have been investigated.
Among the most studied applications is the real-time allocation of radar resources (for example  \cite{wintenby2006hierarchical}). However, other applications such as  multi sensor management \cite{xiong2002multi} have also been studied. One more emerging application is the use of services like Mechanical Turk in order to extract features (for example subjective features regarding an image or a text). The more averaging performed, the more accurate the features are. However,  not all features require the same accuracy.

In our model, collected features are corrupted by some disturbance. We explore two types of disturbances: stochastic and adversarial. A stochastic disturbance corresponds to common situations where features are corrupted by some, typically additive, noise. An adversarial disturbance concerns the worst possible deterministic loss maximizing noise corresponding to ``worst-case'' scenarios.

We assume that special effort is made so that the training data are of the highest quality. During the test phase, however, resources are limited and should be allocated sparingly. This is often the case in  applications where the number of samples to be classified is larger by several orders of magnitude than the training set size.
This work focuses on methods for controlling uncertainty in problems of binary classification with real valued features. 
We consider support vector machines (SVM) style classification \cite{hsu2003practical} due to its many beneficial properties (for example \cite{steinwart2005consistency} and \cite{xu2010distributional}). However, our method can be easily adapted to a wide variety of learning schemes.

 We further explore a second scenario in which we assume that the training data are noisy while during the test phase data quality is superb. This can occur for a number of reasons. One example is some difficulty to gather information in the learning phase which do not exist in the test-phase. For example, patients may be more willing to conduct a  CT scan when some serious illness is suspected but convincing them to perform one for the sake of experimentation require the use of less radiation therefore more noise\cite{cesa2011online}. Another example is when  the learning data-set is ``sensitized'' by artificiality adding noise in order to comply with privacy issues. Scenarios in which noise arise in both training and testing phase can be accommodated by a combination of the methods presented.

In most of the paper we assume that the relation between the resources to be allocated and the disturbance is known. This scenario is quite reasonable, examples include influence of sampling rate on temporal features, sampling time on spectral features, power on channel error rate in communication and many more. However, since there are also cases where this relation is unknown we introduce an algorithm that is \emph{completely}  data-driven. We do not assume Gaussian noise. However, in many areas of control and signal processing Gaussian noise is used to model sensors noise. For that reason the examples given consider Gaussian noise.

%This problem had been studied extensively within the scope of control theory\cite{hero2011sensor}. However much %less attention was given to the issue of resource allocation in classification problems and prior work was either highly %heuristic (for example \cite{goodman2007adaptive}) or based on information measures without incorporating the %classifier (for example \cite{kastella1997discrimination}). The existence of some decision boundary make resource %allocation for classification more involved and existing frameworks of resource allocation inadequate.     

{\bf Related works.} The problem of resource allocation between sensors has been investigated in several disciplines and from several perspectives. Most works come from an adaptive control perspective. Almost half a century ago, Meier \cite{meier1967} defined  a setting where sensors parameters can be controlled. The control perspective has been studied extensively since, mostly for the special case of sensor switching, namely dynamically choosing one sensor from several available ones; see \cite{athans1972determination} and many others. In contrast with those works we are dealing with classification problem. The existence of some decision boundary makes the problem more involved and the control theory framework inadequate. % does not fit well to  used in most of those works.  
 In addition, this line of research generally assumes full knowledge of the underline model, an assumption we would like to avoid. 

 In \cite{avitzour1990optimal} the authors considered the problem of finding an optimal least-squares linear regressor as well as noise parameters of a static estimation problem when the underlying model is known. They explore the spacial case of estimating a scalar using square loss. A mild extension to this spacial case is given in \cite{shakeri1995optimal}. We generally follow the same approach, although our problem definition is  more general. We fortunately have the privilege of enjoying a later  rich body of research concerning dealing with known uncertainty in learning scenarios (e.g., \cite{xu2009robustness,shivaswamy2006second}). 

Classification problems in this context were considered by trying to maximize some measure of information in the data. In this setting one tries to optimize some information measures like sample conditional entropy or the Kullback-Leibler (KL) divergence (for example \cite{jenkins2010adaptive}).  Such methods lead to an elegant solution but are heuristic and ignore knowledge about the desired utility function, so that some information ``quantity'' is optimized instead of the relevance to classification.

Resource efficient learning is a growing field of research in recent years. Most research is focused on dynamic acquisition of features where different features are acquired for different samples. Multiple models were proposed including trees \cite{xu2013cost}, cascades\cite{trapeznikov2013multi} and Markov decision processes \cite{gao2011active}. Our work explores the situation where features are acquired simultaneously and not sequentially.  Some work had also considered introducing resource awareness into the classifier learning process. This is usually done using some greedy process where features are added to a classifier until the resource budget run out\cite{xu2012greedy,nan2015feature}. Similar methods which treat the learning scheme as ``black-box'' are wrapper feature selection \cite{Kohavi1997}. Some work had explored similar issues when resources are scarce in the learning phase instead of the testing phase \cite{lizotte2002budgeted,melville2004active}. While  our work shares a similar motivation with those fields, our decision space is continuous and not discrete. We are inspired by problems in which sensors use a physical resource which need to be allocated (time, power, bandwidth, etc.). Existing methods cannot support such problems. In addition,  the use of a continuous decision space circumvents the need to  solve complex combinatorial problems and allows the use of various tools from optimization theory.  

%Other related fields are  wrapper feature selection \cite{Kohavi1997} and budgeted learning \cite{lizotte2002budgeted,melville2004active}. While  our work shares a similar ideology with those fields, our decision space is continuous and not discrete. In addition we assume that features are acquired simultaneously. The use of a continuous decision space circumvents the need to  solve complex combinatorial problems and allows the use of various tools from optimization theory.

 Another setting which had been explored is on-line learning in the presence of noise. An algorithm for on-line learning from noisy data is presented in \cite{cesa2011online}. We improve the algorithm presented there by allowing on-line control of features quality and show that learning can be done more efficiently.

{\bf Contributions.} 
The contributions of this paper are:
\begin{itemize}
\item We develop a framework for considering feature acquisition quality as a resource allocation problem in classification.
\item We derive algorithms for optimal resource allocation and optimal classification for a variety of scenarios.
\item We analyse the performance gain that can be achieved.
\item We demonstrate the benefit that can arise from using those methods in simulation.
\end{itemize}

The structure of this paper is as following: Section  \ref{section:statistical}  introduces the framework of uncertainty management and provides a method for determining the optimal resource allocation for stochastic disturbances. Section \ref{section:adversarial} explores the case of adversarial noise. The results presented in those sections characterize the optimal allocation for a wide array of problems. %Section \ref{section:algorithmic}  is concerned with some algorithmic challenges of the methods presented and possible solutions.  
Section \ref{sec:unknown} proposes an algorithm for the scenario where the disturbance characteristics is unknown and gives a theoretical guarantee on its regret. Section \ref{sec:learning} explores the case where the training set is noisy and provides an efficient algorithm for the special case of linear classifier with Gaussian noise and square loss. Section \ref{section:simulation}  presents some simulations that demonstrate the feasibility of the results and Section \ref{section:conclusion} concludes with some final thoughts. 
Proofs for all of the theorems in this paper can be found in the appendix. 

\section{Uncertainty allocation: Stochastic disturbances} \label{section:statistical}

This section explores the case in which the disturbance is stochastic.
We assume that M samples $(x,y)\in(\mathbb{R}^d,\{-1,1\})$ are generated from some joint distribution (i.i.d.). 
Denote by $X_{ij}$ the $j$'th feature of sample $i$. Each $X_{ij}$ is measured with some disturbance  $\delta_{ij}$.
 The disturbance is generated from a distribution with some vector of parameters (resources) $r=(r_1,\ldots,r_d)$. Denote the resulting vector of disturbances in sample $i$ as $\delta_i$. 
 %The classification process wish to minimize some cost function $L(h)=\mathbb{E}_{x,y}(l(h,x,y))$.
We follow the empirical risk minimization framework \cite{vapnik1998statistical}.  Let $L(h,r)$ be the cost  incurred when the disturbance is generated using resource vector $r$. That is. 
$$L(h,r)=\frac{1}{M}\sum\limits_{i=1}^{M}\mathbb{E}_{\delta}(l(h,X_i+\delta_i,Y_i)).
$$
Our objective is  to optimize  {\em both} the resource vector $(r_1,\ldots,r_d)$ and the classifier $h(x)$ such that  $L(h,r)$ is minimized. 

For simplicity, we focus our attention on the spacial case of linear classifiers. However, the framework presented  in this paper can be easily extended to other families of classifiers. Also, we assume that the noise is independent between samples, namely that each $\delta_{ij}$  is generated i.i.d. using a distribution with parameter $r_j$. 
We note  that $L(h,r)$ can also be written as $L(h,r)=\tilde{L}(h,\sigma(h,r))$ where $\sigma(h,r)\in \mathbb{R}$. The variable $\sigma(h,r)$ is a measure of the noise influence on the cost function. For example, for linear classifiers it is often the standard deviation of the noise in the axis perpendicular to the decision boundary.    This is helpful since many loss function may be defined this way. We give details of such an example below (Example 1). %(EXAMPLE XXXX).

% For example, if $h$ is a linear classifier $(w_1,\ldots,w_k)$ and the noise is Gaussian with parameters $\sigma_i=\frac{1}{\sqrt{r_i}}$ then  $\sigma=\sqrt{\sum\limits_{i=1}^{k}\frac{w_i^2}{r_i}}$ for a variety of loss functions which depends on the distance from the dividing plane. Examples include the 0-1 loss, the hinge loss, squared loss or any other loss function. 

For a linear classifier we define,
 	
	\begin{equation*}
	L(w,b,r)=\frac{1}{M}\sum\limits_{i=1}^{M} \mathbb{E}_{\delta}(l(Y_i,(w^T(X_i+\delta_i)+b)).
	\end{equation*}
	 Assume that $\sigma(\cdot,\cdot,r)$ is a convex function in $r$, positive and strictly decreasing in each element for $r_j>0$. Also, assume that $\tilde{L}(w,b,\sigma)$ is strictly increasing and convex in $\sigma$. We refer to a loss function that satisfies these assumptions as an \emph{acceptable loss function}. 
	 Those assumption can be informally interpreted as assuming that more resources provide better accuracy and that increasing performance provide diminishing return.
	 % We believe that those assumptions hold in many scenarios. 

	The problem can now be stated as:
		\begin{eqnarray}\label{problem:stat}
				& \min_{r,w,b} & L(w,b,r) \bydef \tilde{L}(w,b,\sigma(w,b,r)) \nonumber\\
				& s.t. 				& \sum\limits_{i=1}^{d} r_i \leq R\\
				&& r_j \geq 0\;\; \forall j . \nonumber
		\end{eqnarray}

\paragraph{Example 1}
Consider the case in which  $\delta_{ij}$ is Gaussian with zero mean and standard deviation $\sigma_j(r_j)$, where $\sigma_j(r_j)$ is a convex strictly decreasing function.
Assume also that  $l(x,y,w,b)=l(w^\top x+b,y)$. In this case, $\sigma$ is the standard deviation of the distance from the decision boundary, namely, $\sigma(w,b,r)=\sqrt{\sum\limits_{i=0}^{d} w_i^2\sigma_i(r_i)^2}$.

Now,  there are two natural loss functions we can explore: hinge loss and square loss.
For the  hinge loss  $l(x,y,w,b)=\max(0,1-y(w^\top x+b))$. 
In such case $L(w,b,r)$ can be calculated directly:    
		\begin{equation*}
		L(w,b,r)=\frac{1}{M}\sum\limits_{i=1}^{M}\left(\frac{1}{\sqrt{\pi}\sigma(w,b,r)}\int\limits_{-\infty}^{1}(1-z)e^{-\frac{(Y_i(w^\top X_i+b)-z)^2}{2\sigma(r)^2}}dz\right).
		\end{equation*}
				
For the square loss $l(x,y,w,b) \bydef(w^\top x+b-y)^2$. 
In this case a simple calculation shows that the overall loss is: 
	    \begin{equation}\label{Lforsquereloss}
	    	L(w,b,r)=\sigma(w,b,r)^2+\frac{1}{M}\sum\limits_{i=1}^{M}(Y_i-(w^\top X_i+b))^2.
	    	\end{equation}
	    		    					
Similarly, one can use  other loss functions and obtain a numerical if not exact expressions.

The following theorem characterizes the optimal  resource allocation for problem (\ref{problem:stat}).
According to Theorem \ref{optimal-noise} the resource allocation depends only on $\sigma(w,b,r)$ such that one can derive the optimal resource allocation even without knowing $L(w,b,r)$. The proof of the theorem and all other proofs appear in the appendix.   
%Notice that finding the optimal allocation require solving a convex minimization problem, as can be seen in the next lemma.

%	We will now state one of the main theorems of this paper which characterize the optimal  resource allocation for problem (\ref{problem:stat})
	\begin{theorem} \label{optimal-noise}
	%[SM UNDER WHAT ASSUMPTIONS ON THE NOISE - STATE CLEARLY!]
Suppose that $L(w,b,\sigma)$ is an acceptable loss function.
	For the optimal solution $(w,b,r)$ of problem (\ref{problem:stat})  there exists   $\lambda>0$ such that
	$\sum\limits_{i=1}^{d} r_i = R$, and for every $i$ it holds that
		\begin{equation}
		\begin{array}{ll}
		r_i=0& \text{if } -\frac{\partial \sigma}{\partial r_i}(w,0)<\lambda\\
		 -\frac{\partial \sigma}{\partial r_i}=\lambda  & else 
		\end{array}.
		\end{equation}
	\end{theorem}

Using Theorem \ref{optimal-noise} and greedy search over $\lambda$ problem (\ref{problem:stat}) can be solved. 
    \subsection{Examples}
    We  now outline a few examples of optimal allocation of resources for different relations between the resources and the noise variance. While all of the examples relate to zero mean Gaussian noise, Theorem \ref{optimal-noise} is general and can be applied for other distributions as long as their variance is finite. % [SM - UNDER WHICH RESTRICTIONS?]
     
    \paragraph{Example 2 - Standard deviation proportional to inverse of resource} 
        
        	We explore the scenario in which the standard deviation is proportional to the inverse of the resources allocated. Namely, $\sigma_i(r_i)=\frac{1}{r_i}$.
        	This is the case, for example, when the resource is the sampling rate and the features measured are timing of various events.
        	In this case:
        	\begin{equation*}
        			r_i=\frac{R w_i^\frac{2}{3}}{\sum\limits_{j=1}^{d} w_j^\frac{2}{3}}.
        			\end{equation*}

    \paragraph{Example 3: Variance proportional to inverse of resources} 
    
    	A popular relation between resources and noise is when the variance is proportional to the inverse of resources allocated. Namely, $\sigma_i(r_i)={1}/{\sqrt{r_i}}$.
    	This is the case in many situations including:  power in active sensors,  duration of sampling for spectral features and  number of measurements taken when averaging (for example if features are extracted using a Mechanical Turk). In this case, the optimal allocation can be easily computed to be:
		\begin{equation}\label{optimal_allocation}
		\hat{r}_i=\frac{R|w_i|}{|w|_1}.
		\end{equation}
	\begin{corollary}\label{coll:ridge}
	In the case of square-loss and uniform allocation of resources $r_i=R/d$, it follows that $\sigma(r)^2 \propto |w|_2^2$.
	\end{corollary}	
	
	\begin{corollary}\label{coll:lasso}
		When applying optimal allocation of resources according to (\ref{optimal_allocation}) it results that $\sigma(r)^2 = |w|_1^2/R $. 
	\end{corollary}	

Interestingly, the optimization problem derived for square-loss (\ref{Lforsquereloss}) with uniform allocation of resources is equivalent to the optimization problem derived when performing ridge regression. Similarly, using optimal allocation of resources is similar to performing lasso regularization. This support claims that using lasso regularization produce classifiers which are more robust to noise than other regularization techniques \cite{xu2009robustness} 

			Since for the square-loss optimizing (\ref{Lforsquereloss}) is equivalent to performing lasso, one can use complicity bounds derived for this case.
			It is known that a bound on the error resulting from lasso regularization $|w|_1<B$ is increasing in $B$ \cite{kakade2009complexity}.
			Since $B$ is decreasing with $1/R$, it is increasing with $R$. This surprisingly  implies that 	{\em less resources require less examples to learn}. This can be explained in the following manner: with less resources there is more noise in the decision making phase, the larger the noise the less impact small changes in the classifier makes (in the limit, there are no resources and therefore the noise is infinite and there is nothing to learn). 
     
     \paragraph{Example 4: Quantization noise} 
     		
     		It is known that rounding quantization noise can be treated as Gaussian with standard deviation of $\frac{1}{12}LSB$ where the $LSB$ is the accuracy of the least significant bit \cite{stein1967modern}.  %[SM: KNOWN TO WHOM? GIVE A REF].
     		Consider a scenario in which we would like to maintain the number of bits used to represent all features under some threshold $R$. We will first disregard the fact that $r$ must be an integer and derive the solution for   $\sigma_i(r)=2^{-(r_i)}$ while $r_i \geq 1$ for all $i$. 
     		The solution of (\ref{problem:stat}) for any fixed $w$ will be  
     		
     		\begin{equation*}
     			\begin{array}{l}
     			r_i=\bigg\{
     			\begin{array}{ll}
     					1 & \log|w_i|<\lambda\\
     					1+\log|w_i|+\frac{1}{|C|}(R-d-\sum_{i\in C} \log|w_i|) & \log|w_i| \geq \lambda\\
     					
     			\end{array}\\
     			\lambda = \frac{1}{|C|} (\sum_{i\in C} \log|w_i|)-R+d)\\
     							C=\{i | \log|w_i| \geq \lambda\}.
     							
     			\end{array}	
     		\end{equation*}
     	\orans
     		Notice that we still need to transform $r_i$ into integers. This can be done by ``searching'' in the vicinity of the optimal vector $r$.
     		% We will do that by truncating $r_i$ and dividing the "extra" resources by sorting the reminder from the truncation procedure. 		
		\orane
     	
%	\begin{corollary}
%			Since for the square-loss optimizing (\ref{Lforsquereloss}) is equivalent to performing lasso, one can use complicity bounds derived for this case.
%			It is known that a bound on the error resulting from lasso regularization $|w|_1<B$ is:
%			\begin{equation*}
%			L_D(\hat{w}).\leq \min_{|w|_1<B}L_S(w)+cB^2\sqrt{\frac{1}{m}\ln\frac{d}{\lambda}}
%			\end{equation*}
%			Since $B$ is decreasing with $\frac{1}{R}$ it is increasing with $R$. This implies the surprising result that less resources require less examples to learn. This can be explained in the following manner: Less resources means more noise in the decision making phase, the larger the noise the less impact small changes in the classifier makes. Therefore less complicity.    
%     \end{corollary}	
%		

\subsection{Performance analysis}

To gain some insight about the expected benefit of using this method  we explore the special case of square loss with $\sigma_i(r_i)= 1/\sqrt{r_i}$. Observe that $L$ is decreasing in $R$.
From Corollary \ref{coll:ridge} and \ref{coll:lasso} we know that finding an optimal $w$ for uniform allocation is equivalent to performing ridge regression while optimizing (\ref{Lforsquereloss}) is equivalent to performing lasso. 
%The benefit from both kinds of regularization techniques had been studied extensively. Since $L$ is decreasing in $R$ 
% we wish to establish how much less optimally allocated resources are needed  to preserve the same error rate as when resources are uniformly allocated. 
We ask the following question: for the same expected loss how much resources can we save by using the method presented?  In order to answer this question we start by fixing $w$ and analyze the expected loss for different resources allocations.

For every admissible $(w,b)$ denote by $R_{unif}(w,l)$ the resource budget that holds $L(w,b,r_{unif})=l$ when $r_{unif}=(R/d,\ldots,R/d)$.
Also, denote by $R_{opt}(w,l)$ the resource budget that holds $L(w,b,r_{opt}(w))=l$ when $r_{opt}(w)=(R|w_1|/|w|_1,\ldots,R|w_d|/|w|_1)$.
The following result bounds the ratio between resources required for achieving the same loss. 
\begin{theorem}\label{th:R_improve}
For every $w$ and for $l(x,y,w,b)=(w^\top x+b-y)^2$  it holds that $\frac{R_{unif}(w,l)}{R_{opt}(w,l)}=\frac{d|w|_2^2}{|w|_1^2}$.
\end{theorem}

\orans
The proof can be found in the appendix.

Denote by $w_{opt}$ the optimal classifier when resources are allocated optimally and by $w_{unif}$ the optimal classifier when resources are allocated uniformly.
The next corollary follows directly from Theorem \ref{th:R_improve}; it bounds the total benefit that can arise from the joint optimization of both resource allocation and classifier.
It holds since  $L(w_{unif},r_{unif})\leq L(w_{opt},r_{unif}) $ and  $L(w_{opt},r_{opt}(w_{opt}))\leq L(w_{unif},r_{opt}(w_{unif}))$.

\begin{corollary} \label{cor:performence}
For every $w$ and for $l(x,y,w,b)=(w^\top x+b-y)^2$  it holds that
\begin{equation*}
\frac{d|w_{unif}|_2^2}{|w_{unif}|_1^2} \leq \frac{R_{unif}(w_{unif},l)}{R_{opt}(w_{opt},l)} \leq \frac{d|w_{opt}|_2^2}{|w_{opt}|_1^2}.
\end{equation*}
\end{corollary}

It is clear from Corollary \ref{cor:performence} that in cases where some features hold little information (small coefficients in the classifier) the benefit of optimized resource allocation can be very large. It should be noted that in extreme cases this is equivalent to using feature selection (meaning, choosing which features should be allocated zero resources). However, in many cases even when considering only relevant features the variance of their influence is significant. In such cases our method provides considerable benefit.

\orane

\section{Adversarial disturbance} \label{section:adversarial}
We  now  consider the case where the disturbance is adversarial. Several models for adversarial disturbance have been considered in the literature, we will adopt the model from  \cite{xu2009robustness}.  
Formally, consider some samples  $\{(X_i,Y_i)\}_{i=1}^{M}$ where $X_i\in  \chi\subset\mathbb{R}^d$ and $Y_i\in\{-1,1\}$. We only have access to some corrupted version of this set $\{(X_i+\delta_i,Y_i)\}_{i=1}^{M}$.  The disturbances $\delta_i$ are determined by an adversary, however the adversary can only affect samples in a certain way. Formally, the vector $\delta=(\delta_1,\ldots,\delta_M)$ is in a set defined by :
% \begin{equation}
$\mathcal{N}(\mathcal{N}_0)\bydef\{ (\alpha_1 \delta_1,\ldots,\alpha_M\delta_M ) | \delta_i \in \mathcal{N}_0 \; \text{for } i=1,\ldots,M, \sum\limits_{j=1}^{M} \alpha_j=1  \}$,
% \end{equation}
where $\mathcal{N}_0$ is some symmetric uncertainty set that contains the origin. 
% For a given $\mathcal{N}_0$ The optimal classifier is the solution of the problem 
%\begin{equation}
% \min_{w,b} \sup\limits_{\delta \in \mathcal{N}}  \sum\limits_{i=1}^{M} L(w,b,X_i+\delta_i,Y_i)
%\end{equation}

In our setting, we wish to optimize both the classifier parameters $w,b$  and the shape of $N_0$ under the constraint of available power (or budget) for the adversary. The main difference here from other works (see \cite{xu2009robustness} and follow-ups) is that we can {\em optimize} over $\mathcal{N}_0$ out of a family of sets (set of sets). Such a family can be,  for example, the set of ellipsoid sets while maintaining some constant fixed resource budget. 
\begin{equation*} 
\mathcal{N}^{set}=\bigg\{ \mathcal{N}_0=\big\{ x | \sum\limits_{i=1}^{d}{(\frac{x_i}{\sigma_i(r_i)})^2} \leq 1 \big\}, \;\;\;\; \sum\limits_{i=0}^{d} r_i =R \bigg\}.
\end{equation*}
 
Formally, the problem is optimizing
\begin{equation}\label{optimization-adv}
\inf\limits_{\mathcal{N}_0 \in \mathcal{N}^{set}} \sup\limits_{\delta \in \mathcal{N}(\mathcal{N}_0)} \min_{w,b} L(w,b,X+\delta,Y)
\end{equation}
for some $\mathcal{N}_{set}$ that defines the problem. 
We will focus our attention on the hinge loss,    
%\begin{equation} \label{non linear_cost}
$L\bydef\sum\limits_{i=1}^{M}  \max(0,(1-y_i(<w,X_i>+b)))$.
%\end{equation} 

For hinge loss the following result is given in \cite{xu2009robustness}:
\begin{lemma}{(Xu et al. 2009\cite{xu2009robustness})} \label{xu_lemma}
Assume $\{X_i,Y_i\}_{i=1}^{M}$ are non-separable then the following min-max problem
\begin{equation*}
\min\limits_{w,b} \sup\limits_{\delta \in \mathcal{N}} \sum\limits_{i=1}^{M}  \max(0,(1-y_i(<w,X_i>+b)))
\end{equation*}
is equivalent to the following optimization problem
$$
\begin{array}{lllll} 
& & \min\limits_{w,b,\xi}
\sup\limits_{\delta \in \mathcal{N}_0 } (w^T\delta)+\sum\limits_{i=1}^{M} \xi_i & &\\
& & s.t.  & &\\ 
& & \xi_i \geq 1-y_i(w^Tx_i+b), & \xi_i \geq 0, &  i=1\ldots,M  \,.
\end{array}
$$
\end{lemma}

We use this result in order to derive the following theorem:
\begin{theorem}\label{theorem_adversrial}
Consider the solution $(\hat{N}_0,\hat{w} ,\hat{b})$  for the problem 
\begin{equation} \label{problem1}
\inf\limits_{\mathcal{N}_0 \in \mathcal{N}^{set}}\min\limits_{w,b}\sup\limits_{(\delta_1,···,\delta_M)\in \mathcal{N}}{
L(w,b,X+\delta_i),Y)}.
\end{equation}
Then the solution satisfies
$$
\hat{N}_0 \in \arg\min_{\mathcal{N}_0\in \mathcal{N}^{set}}\sup\limits_{\delta \in \mathcal{N}_0 } (w^T\delta).
$$
\end{theorem}

\orans
Theorem \ref{theorem_adversrial} is analogous to Theorem \ref{optimal-noise} and allows to optimize resource allocation in the adversarial setting.  
\orane

\paragraph{Example 4}

 Gaussian noise is a popular modelling choice in many domains. We wish to find some constraint which will create in the adversarial setting an effect that reassembles Gaussian noise. For this purpose we use an ellipsoid uncertainty set. Instead of assuming Gaussian noise we bound the uncertainty to a fixed width of standard deviations.
Consider the model presented in \cite{xu2009robustness} with an ellipsoid uncertainty set namely, $\mathcal{N}_0=\{ x | \sum\limits_{i=1}^{d}{({x_i}/{\sigma_i(r_i)})^2} \leq 1 \}$. The function $\sigma_i(r_i)$ can be any of the former examples.

% the problem
%\begin{equation} 
%\min\limits_{w,b}\sup\limits_{(\delta_1,···,\delta_M)\in \mathcal{N}}{
%\sum\limits_{i=1}^{M}{max(1-y_i(w(x+\delta_i)+b), 0)}}
%\end{equation}
%\begin{equation} 
%\mathcal{N}_0=\{ x | \sum\limits_{i=1}^{d}{(\frac{x_i}{\sigma_i(r_i)})^2} \leq 1 \} \\
%\end{equation}
% \begin{equation}
%\mathcal{N}=\{ (\alpha_1 \delta_1,\ldots,\alpha_M\delta_M ) | \delta_i \in \mathcal{N}_0 \; i=1,\ldots,M, \sum\limits_{j=1}^{M} \alpha_j=1  \}
% \end{equation}

%Suppose now one can control the uncertainty set, the following theorem provides a method for solving the problem. 
%namely the problem:
%\begin{equation} 
%\min\limits_{\mathcal{N}_0}\min\limits_{w,b}\sup\limits_{(\delta_1,···,\delta_M)\in \mathcal{N}}{
%\sum\limits_{i=1}^{M}{max(1-y_i(w^T(x+\delta_i)+b), 0)}}
%\end{equation}
%s.t 
%\begin{equation}
%\sum\limits_{i=0}^{d} r_i =R
%\end{equation}

%\begin{theorem}
Now, under non separability assumption the  solution of the problem 
$$
\begin{array}{llll}
\min\limits_{r} &\min\limits_{w,b}\sup\limits_{(\delta_1,···,\delta_M)\in \mathcal{N}(r)}{
\sum\limits_{i=1}^{M}{\max(1-y_i(w^T(x+\delta_i)+b), 0)}}, \; & \text {s.t.}  & \sum\limits_{i=0}^{d} r_i =R
\end{array}
$$

satisfies,
$$
\begin{array}{llllll}
w_i^2\sigma_i(r_i)\frac{d\sigma_i(r_i)}{dr_i}&=&\lambda,&
\sum\limits_{i=1}^{d}r_i&=&R.
\end{array}.
$$

Fixing $\sigma_i$ allows the derivation of  $w,b$  by solving the conic optimization problem
\begin{equation*} %\label{problem1_4}
\begin{array}{lll}
\min\limits_{w,b,\xi}
\sqrt{\sum\limits_{i=1}^{d} w_i^2\sigma_i^2} +\sum\limits_{i=1}^{M} \xi_i &&\\
s.t&& \\ 
\xi_i \geq 1-y_i(w^Tx_i+b),&\xi_i \geq 0, & i=1\ldots,M.
\end{array} 
 \end{equation*}
 
Theorems \ref{optimal-noise} and \ref{theorem_adversrial} allow to solve either (\ref{problem:stat}) or (\ref{optimization-adv}) using alternating optimization. 

Moreover, in the special case where $\sigma_i(r_i)={1}/{\sqrt{r_i}}$ the optimal allocation of resources is analogous to lasso regression:
\begin{equation} \label{lasso}
\sqrt{\sum\limits_{i=1}^{d} w_i^2\sigma_i^2}=\frac{|w|_1}{\sqrt{R}},\qquad
r_i=\frac{R|w_i|}{|w|_1}\text{ for } i=1,2,\ldots,d\,\,.
\end{equation}

\section{Unknown stochastic disturbance} \label{sec:unknown}
In this section we consider the case of stochastic disturbance that is unknown.
We wish do devise a data-driven algorithm that finds the optimal resource allocation even when the disturbance is initially unknown. We use stochastic gradient descent in order to minimize the cumulative loss function.  In this section we explore the special case of square-loss with some assumptions on the structure of the disturbance. We derive a concrete algorithm and a corresponding bound for this special case. It is possible to easily extend this algorithm to various other scenarios.
\orans
We make the following assumptions on the structure of the disturbance:
\begin{itemize}
\item  The disturbance and data-points are independent ($X_i$ is independent from $\delta_i$).
\item  The  disturbance is independent between features.
\item  The  distribution of the disturbance in each feature is symmetric.
\item  The second moment of the disturbance, $\sigma_i^2(r_i)$ is convex in $r_i$.
\end{itemize}
\orane
The last assumption is reasonable since we expect diminishing return from increasing  allocated resources.   
We use ridge regularization and bound the possible set of classifiers by $||w||_2\leq B_w$.

The optimization problem can be stated as:
\begin{eqnarray}\label{problem:unknown}
				& \min_{r,w} &\sum\limits_{t=1}^{T}l(w,r,x^t,y^t) \bydef \sum\limits_{t=1}^{T}\sum\limits_{i=1}^{d} w_i^2\sigma_i^2(r_i)+(w^Tx^t-y^t)^2 \nonumber\\
				& s.t. 				& \sum\limits_{i=1}^{d} r_i = R\\
				&& r_j \geq 0\;\; \forall j  \nonumber\\
				&& ||w||_2\leq B_w.
		\end{eqnarray}

The gradient is given by:
\begin{eqnarray}\label{gradient}
				\frac{\partial l}{\partial w_i}& =& 2w_i\sigma_i^2(r_i)+2x_i(w^Tx-y) \nonumber\\
				\frac{\partial l}{\partial r_i}& =&w_i^2\frac{\partial \sigma_i^2(r_i)}{\partial r_i}. \nonumber
		\end{eqnarray}

Since $\frac{\partial \sigma_i^2(r_i)}{\partial r_i} $ is unknown we will approximate it using the Kiefer-Wolfowitz procedure \cite{kushner1997stochastic}. This results in  $\frac{\partial \sigma_i^2(r_i)}{\partial r_i} \approx \frac{\sigma_i^2(r_i+\epsilon)-\sigma_i^2(r_i)}{\epsilon}$.  
We will denote by $\Pi(w,r)$ the projection of classifier $w$ and resource vector $r$ into the set of feasible solutions $\mathcal{N}\bydef\{|w|_2<B_W,\sum r_i=R,r_i>0\}$. We further denote the maximum distance between two vectors in this set by $B=2\sqrt{R^2+B_W^2}$. 

It is now possible to use standard stochastic gradient descent. Following the Kiefer-Wolfowitz procedure,  at each step  measure two data points with two slightly different resource allocations. Then, estimate the gradient and update the classifier and resource allocation accordingly. Finally, project the solution into the feasible solutions space and continue to the next step. The resulting algorithm is presented as Algorithm \ref{alg:unknown}.
\begin{algorithm}[tb]
   \caption{Learning when the disturbance is unknown}
   \label{alg:unknown}
\begin{algorithmic}
   \STATE {\bfseries Parameters}    $B_W,R,\epsilon$  
   \STATE initialize $w_1=0$, $r^1_i=R/d$
   \FOR{$t=1,2,3,\ldots,T$}
   \STATE receive $\hat{x}^{t,1},y^{t}$ using resource distribution $r^t$  
   \STATE receive $\hat{x}^{t,2},y^{t}$ using resource distribution $r^t+\epsilon$      
   \STATE $\eta=1/sqrt(t)$
   \FOR{i=1,\ldots,d} 
   \STATE $w_i'=w_i^t-\eta((<w^t,\hat{x}^{t,1}>-y^{t,1})\hat{x}^{t,1}_i)$
   \STATE $r_i=r^t_i-\eta\frac{{(w^t_i)}^2[({\hat{x}_i^{t,2})}^2-{(\hat{x}_i^{t,1})}^2]}{\epsilon}$      
   \ENDFOR
   \STATE $(w^{t+1},r^{t+1})=\Pi(w,r)$\;\;\;;$\Pi(w,r)$ is the projection into the feasible solutions space 
   \ENDFOR
   \end{algorithmic}
\end{algorithm}

It is easy to verify that the estimated gradient is indeed unbiased.
Notice that unlike standard on-line learning the measurement $x_n$ are not i.i.d. since choosing $r$ creates a coupling between measurements. However, the ``noise'' of the estimated gradient is a martingale difference sequence and therefore stochastic estimation theory can be easily applied.
%The convergence of this algorithm derives directly from the properties of stochastic gradient descent since the noise is a martingale difference.

We proceed to bound the regret which arise from algorithm \ref{alg:unknown}. Since we use Keifer-Wolfowitz procedure the regret must be measured in comparison to the biased functions created by the procedure. Namely, $\tilde{\sigma}^2_i(r)=\int\limits_{0}^{r}\frac{\sigma_i^2(s+\epsilon)-\sigma_i^2(s)}{\epsilon}ds$ and $\tilde{l}(w,r,x,y) \bydef \sum\limits_{i=1}^{d} w_i^2\tilde{\sigma}_i^2(r_i)+(w^Tx-y)^2$. When $\epsilon$ is small enough $\tilde{\sigma}^2_i(r)$ is approximately $\sigma_i^2(r)$.  
It is now possible to derive a bound on the regret.
\begin{theorem} \label{th:unknown}
If $\tilde{l}(w,r,x,y)$ is jointly convex in $w,r$ for every $x,y$, $\mathbb{E}(x)=0$, $\mathbb{E}(||x||_2^2)=1$ , $\mathbb{E}(||x||_2^4)=B_x^4$, $\mathbb{E}((\hat{x}_i-x_i)^2) \leq B_\delta^2$ and $\mathbb{E}((\hat{x}_i-x_i)^4) \leq B_\delta^4$  then
\begin{equation*}
\mathbb{E}(\sum\limits_{i=1}^{T}(w^tx^{t,1}-y^t)^2)-\min_{(w,r)\in \mathcal{N}}\sum\limits_{i=1}^{T}\tilde{l}(w,r,x^t,y^t)\leq \frac{B\sqrt{T}}{2}+(\sqrt{T}-\frac{1}{2})||\nabla l||^2.
\end{equation*}

Where, 
\begin{equation}
\begin{array}{l}
B=2\sqrt{R^2+B_W^2} \\
||\nabla l||^2=2B_w^2B_{\tilde{x}}^4+2B_{\tilde{x}}^2+\frac{2B_{\tilde{x}}^4B_W^4}{\epsilon^2} \\
B_{\tilde{x}}^4=B_x^4+6B_x^2B_\delta^2+B_\delta^4\\
B_{\tilde{x}}^2=B_x^2+B_\delta^2
\end{array}
\end{equation}

\end{theorem}

The proof follows similar lines to that used to derive a bound in \cite{cesa2011online}  and can be found in the appendix.

Theorem \ref{th:unknown} implies that the optimal classifier and optimal resource allocation can be learned with sub-linear regret. Note that decreasing $\epsilon$, which is the step-size used to estimate the gradient, will increase learning time. This is since we assume that noise is independent between samples. In this setting decreasing $\epsilon$ increases the noise level in estimating the gradient. Choosing large $\epsilon$, however, can result in large bias from the optimal solution. The next two remarks show that assuming some dependence between samples may \emph{reduce}  learning time significantly. 

\begin{remark}
The term $\frac{2B_{\tilde{x}}^4B_W^4}{\epsilon^2}$ can be quite large.
For reducing the variance in the learning process it is possible at some cases during training to sample multiple times the same data point . In such cases it is possible to derive a much better bound in which $||\nabla l||^2=2B_w^2B_{\tilde{x}}^4+2B_{\tilde{x}}^2+\frac{2B_\delta^4B_W^4}{\epsilon^2}$. 
\end{remark}
\orans
\begin{remark}
In many cases the measurements noise of the same sample with different resources is correlated. This is for example the case when the resource is CPU time and the disturbance is caused from  processing only part of the data. Two acquisitions of the same sample share a vast amount of common data.
  In such cases the difference between measurements with $r+\delta$ and $r$ can be bounded much more tightly then the bound used in Theorem \ref{th:unknown}. If $\frac{{(\hat{x}_i^{t,2}-x_i^t)}^2-{(\hat{x}_i^{t,1}-x_i^{t}}^2)}{\epsilon} \leq B_{grad}$ then $||\nabla l||^2$ in Theorem \ref{th:unknown} can be rewritten as $||\nabla l||^2=2B_w^2B_{\tilde{x}}^4+2B_{\tilde{x}}^2+2B_W^4B_{grad}^2$.  
\end{remark}
\orane

\section{Learning from noisy data}\label{sec:learning}

In this section we explore the situation where the learning set is noisy while the test set is of perfect quality. This is the case in certain medical examinations where in the learning phase  it is difficult to persuade a subject to go through extensive testing while at test time a patient suspected of having a serious disease will agree to such testing \cite{cesa2011online}. We adopt the framework in \cite{cesa2011online} that considered learning from noisy data. In our setting, however, the noise distribution can be controlled (under some resource constraints) by the learner. As we will show this control can produce a more efficient learning process.
The on-line learning scheme fits this scenario since the optimal noise allocation depends on the classifier .
We will focus our attention on the case of squared-loss.  In \cite{cesa2011online} the authors develop an algorithm for online learning from noisy data. Their algorithm uses stochastic gradient descent in order to optimize the expected loss.
Our algorithm is a modification of the one presented in \cite{cesa2011online} to include the control over resources. We will use lasso regularization in order to bound the set of classifiers, namely $|w|_1<B_w$. The algorithm is presented as Algorithm \ref{alg:efficientlearn}. The algorithm receives as input the step size $\eta$, the lasso parameter $B_w$ and some function which assign optimal resources for a known classifier $r(w)$. Examples for possible $r(w)$ had been given in section \ref{section:statistical}. The covariance matrix of the disturbance which results from using resources vector $r$ is denoted by $\Sigma(r)$. Notice that $\Sigma(r)$  is diagonal and assumed known.    
We focus on the case where the disturbance is Gaussian with standard deviation $\sigma_i(r_i)=\frac{1}{\sqrt{r_i}}$.

\begin{algorithm}[tb]
   \caption{Efficient learning from noisy data}
   \label{alg:efficientlearn}
\begin{algorithmic}
   \STATE {\bfseries Parameters}    $\eta,B_W,R,r(w)$  
   \STATE initialize $w_1=0$, $r^1_i=R/d$
   \FOR{$t=1,2,3,\ldots,T$} 
   \STATE receive $x_t,y_t$ using resource distribution $r^t$  
   \STATE $\nabla_t=2(<w_t,x_t>-y_t)x_t-\Sigma(r^t)w_t$
   \STATE $w'=w_t-\eta\nabla_t$   
   \STATE $w_{t+1}=\arg\min_{|u|_1<B_W}||w'-u||_2$
   \STATE $r^{t+1}=r(w_{t+1})$
   \ENDFOR
   \end{algorithmic}
\end{algorithm}

Our results are based on the following lemma which is an adaptation of Theorem 2 from \cite{cesa2011online}.
\begin{lemma}
Assume $\mathbb{E}_t(||\nabla_t||^2)\leq G$ and set $\eta=B_w/\sqrt(T)$ then the regret of Algorithm \ref{alg:efficientlearn} satisfy $\mathbb{E}(\sum\limits_{i=1}^{T}(w_t^Tx_t-y_t)^2)-\min_{|w|_1<B_W} (\sum\limits_{i=1}^{T}(w^Tx_t-y_t)^2\leq \frac{1}{2}(G+1)B_W\sqrt{T}$.
\end{lemma}  
Since the proof of this lemma is very similar to the one used to produce the results in \cite{cesa2011online} we refer the reader to \cite{cesa2011online}.

We  now move on to show that a proper choice of resources may improve learning.
We assume the problem is normalized such that $\mathbb{E}(y)=0$, $\mathbb{E}(y^2)=1$,$\mathbb{E}(x)=0$ and $\mathbb{E}(||x||_2^2)=1$. We further denote $\mathbb{E}(||x||_2^4)=B_x^4$. 
The following two theorems show that proper allocation of resources can improve the efficiency of learning by $O(d)$. More specifically the regret will be $O(B_w^3d^2,\sqrt{T})$ instead of $O(B_w^3d^3,\sqrt{T})$.

\begin{theorem} \label{th:uniform}
Assume $r^t_i(w)=\frac{R}{d}$ and $\eta=B_w/\sqrt(T)$. Then
\begin{equation*}
\mathbb{E}(\sum\limits_{i=1}^{T}(<w_t,x_t>-y_t)^2)-\min_{|w|_1<B_W} (\sum\limits_{i=1}^{T}(<w,x_t>-y_t)^2\leq \frac{1}{2}(G+1)B_W\sqrt{T}.
\end{equation*}

Where, 
  
\begin{equation*}
G=32B_w^2\frac{d^3}{R^2}+98B_w^2\frac{d^2}{R^2}+32B_W^2\frac{d^2}{R}+32B_W^2\frac{d}{R}+16\frac{d^2}{R}+32B_W^2B_x^4+16=O(B_W^2d^3)
\end{equation*}  
\end{theorem}

However, in case resources are allocated efficiently the corresponding bound is given by the following theorem.

\begin{theorem} \label{th:efficent}
Assume $r^t_i(w)=\frac{R}{2d}+\frac{R{w_t}_i}{2|w_t|_1}$ and $\eta=B_w/\sqrt(T)$ then
\begin{equation*}
\mathbb{E}(\sum\limits_{i=1}^{T}(w_tx-y)^2-\min_{|w|_1<B_W} (\sum\limits_{i=1}^{T}(<w_t,x>-y)^2)\leq \frac{1}{2}(G+1)B_W\sqrt{T}
\end{equation*}
 Where,
 \begin{equation*}
G=64\frac{d^2}{R^2}B_W^2+64\frac{d^2}{R}B_W^2+32\frac{d^2}{R}+392\frac{d}{R^2}B_W^2+64\frac{B_W^2}{R}+32B_W^2B_x^4+16=O(B_W^2d^2)
 \end{equation*} 
  \end{theorem}

Notice that efficient learning requires some balance between two terms. The term $\frac{R}{2d}$ is required for estimating $\mathbb{E}(x)$ while the term $\frac{R{w_t}_i}{2|w_t|_1}$ is required for estimating $\mathbb{E}(w^Tx)$. We have created $r_(w)$ by balancing those two terms evenly. It is possible that a different balance will provide better results.    

\begin{remark}
%When $w$ is sparse and $|w|_1 \approxeq |w|_2$ then the two bounds hold and appropriate allocation of resources result in $o(d)$ less regret. 
When $w$ is dense the efficient allocation is almost uniform. Therefore, the regret of the two resources allocation schemes should be similar. This is not evident from the bounds provided. The reason is that the proof of Theorem \ref{th:uniform} uses the fact that in the worst case $|w|_2=|w|_1$. In cases where $w$ is dense this is loose. Using a tighter bound,  $|w|_2\approxeq \frac{|w|_1}{\sqrt{d}}\leq \frac{B_w}{\sqrt{d}}$ results in a bound with order $O(B_w^2d^2)$ for the uniform allocation case, similar to that received for efficient allocation of resources. 
\end{remark}

\orane

\section{Simulation study} \label{section:simulation}

We  tested the method on three datasets, one synthetic and two real-life problems from the UCI repository. Noise was added to all data artificially according to the relation $\sigma_i=\frac{1}{\sqrt{r_i}}$.
For all datasets, measurement noise was created using the normal distribution with parameters $(0, \frac{\sigma_i}{3}) $ and was added to the test samples.
We applied the algorithm from the previous section to derive both an optimal classifier and an optimal resource allocation. The result given in Eq.~(\ref{lasso}) was used to derive the optimal resource allocation for a fixed classifier. We used hinge-loss as the loss function to be minimized and approximated $L(w,b,r)$ by using an adversarial ellipsoid uncertainty-set. Optimization was performed using the commercially available Mosek solver \cite{mosek}. 

{\bf Synthetic problem.} We generated 240000 samples uniformly distributed in a box in $\mathbb{R}^3$.  We  used $z=x+7y$ as the divider and created a data-set with labels that obey $sgn(z-x-7y+N)$, where $N$ is some small Gaussian noise we added in order to make the data-set non-separable.  %  added some Gaussian noise in order to make the set non separable. % [SM: EXPLAIN NOISE GENERATION] 
A random subset of 10000 samples was used for learning while performance was measured on the rest. Tenfold cross validation was performed.
 The result for different $R$ values is depicted  at Figure \ref{error_for_simulated}. The method results in about 50\% reduction in resources required for meeting the same error rate. In this case, the optimal classifier is similar to the classifier derived without noise and the benefit arise mainly from the redistribution of noise.

 We wish to confirm the result of Theorem \ref{th:R_improve} using similar synthetic data-sets. For this purpose, we have generated nine data-sets each using as a divider $z=x+ay$ for $a=1,2,\ldots,9$. For each data-set we have extracted the resources needed for achieving an error rate of $0.15$. We calculated the ratio between the total resources required when resources are allocated optimally and those required when resources are allocated uniformly. When $a=1$ the optimal allocation is uniform and we expect no benefit (the ratio equals one). As we increase $a$, more resources should be allocated to $y$ and therefore the ratio is improving (decreasing).   Figure \ref{ratio_of_improvment} shows the resulting graph compared with the theoretical result of Theorem \ref{th:R_improve} (using the optimal classifier). It can be seen that the simulation result is almost identical to the theoretical one, though contrary to the assumptions of Theorem \ref{th:R_improve}  we are optimizing the hinge-loss and measuring error-rate. Observe in Figure \ref{ratio_of_improvment} that considerable benefits arise even when the differentiation between features is rather small. 
%[SM: I DON'T UNDERSTAND WHY FOR SMALL a THE RATIO IS LARGER - CAN YOU EXPLAIN?] [ Oran: The larger the ratio the smaller the benefit, the x axis start at 1 when all features are have equal weight ]

\begin{figure}
   \centering
  
   \centering
  \captionsetup{justification=centering,margin=1cm}
   \includegraphics[angle=0,
      width=0.5\linewidth,height=2.5in]{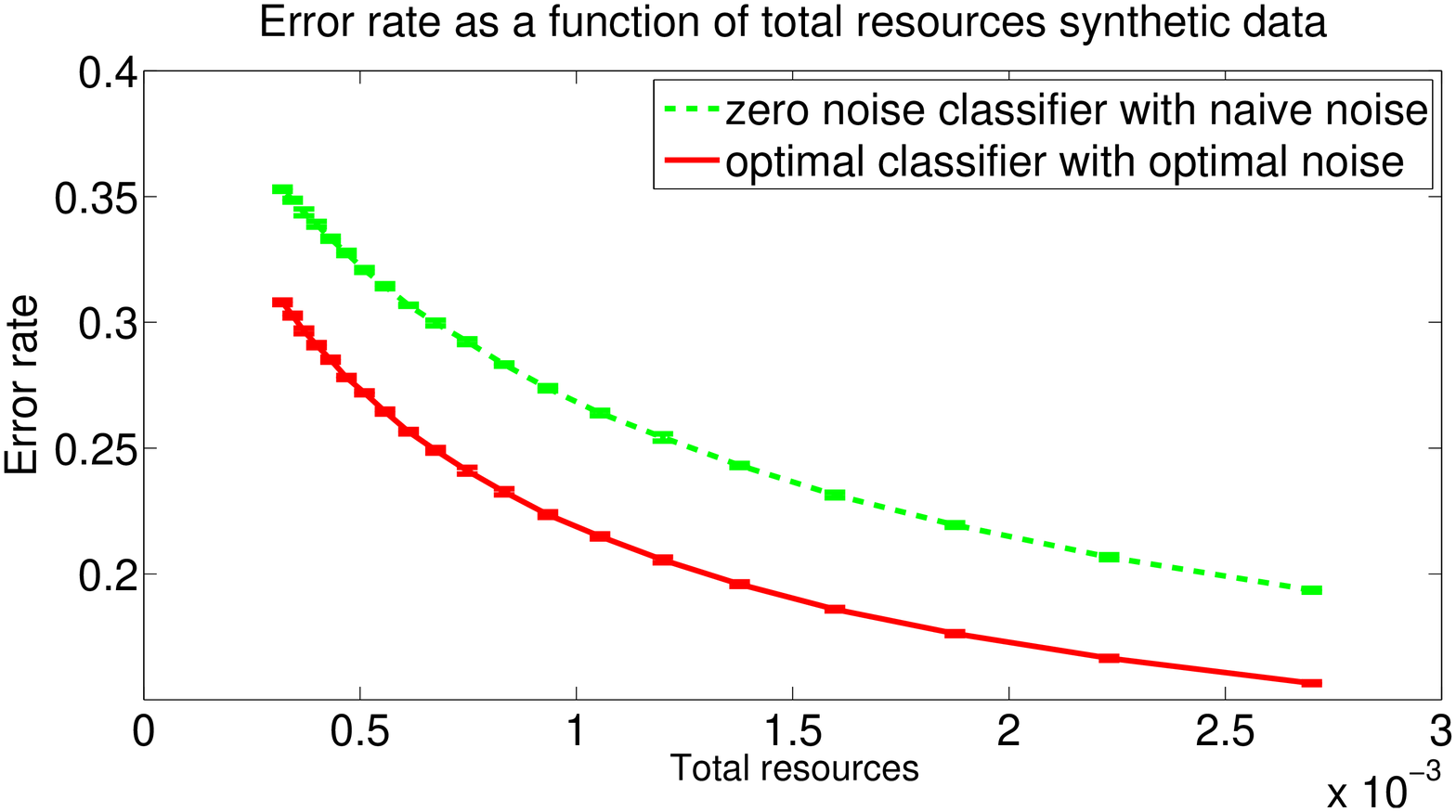}
      \caption{Error rate for synthetic data. }
      \label{error_for_simulated}
   \end{figure}
 \begin{figure}
 
   \centering
 
  \captionsetup{justification=centering,margin=0.5cm}
       \includegraphics[angle=0,
       width=0.5\linewidth,height=2.5in]{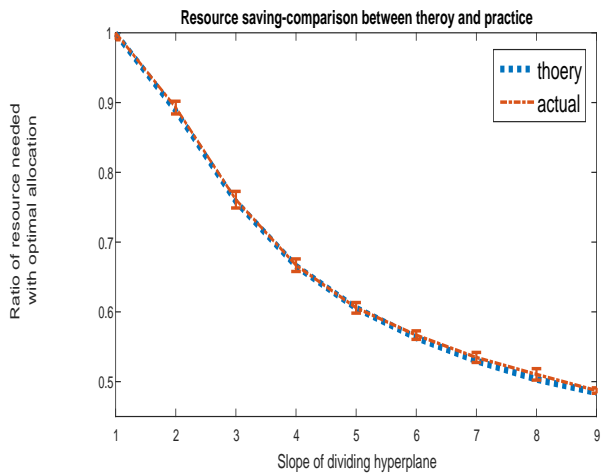}
       \caption{Ratio of resources needed for the same error level- synthetic data. Lower is better}
       \label{ratio_of_improvment}
    %\end{minipage}
   
   \end{figure}

%\begin{figure}\label{error_for_simulated}
%\centering
%\includegraphics[angle=0,
%width=0.8\textwidth]{error_synthetic}
%\caption{Error rate for synthetic data.}
%\end{figure}
%
%\begin{figure}\label{ratio_of_improvment}
%\centering
%\includegraphics[angle=0,
%width=0.8\textwidth]{ratio}
%\caption{Ratio of resources needed for the same error level- synthetic data}
%\end{figure}

{\bf Real data sets.} Next, we tested the method on real-life databases from the UCI repository. We started with the skin segmentation data set \cite{skin} where RGB pixels are classified as skin or non-skin. Noise was added artificially to each pixel From the 245057 available samples, a random subset of 10000 was used for learning while the rest was used to estimate performance. Ten-fold cross validation was performed. The results for different $R$ values can be seen in Figure \ref{error_for_skin}. It can be seen that the method results in about 30 \% reduction in resources. % required to meet the same error rate.
%[SM: NOT CLEAR WHAT PROCESS WAS USED FOR SKIN DATA: DID YOU ADD NOISE TO EACH PIXEL?]

%[SM: FOR ALL FIGURES MAKE THE FONTS LARGER AND THE PLOT ITSELF THICKER]
\begin{figure}
   \centering
   \centering
   \captionsetup{justification=centering,margin=0.5cm}
   \includegraphics[angle=0,
      width=0.5\linewidth,height=2.5in]{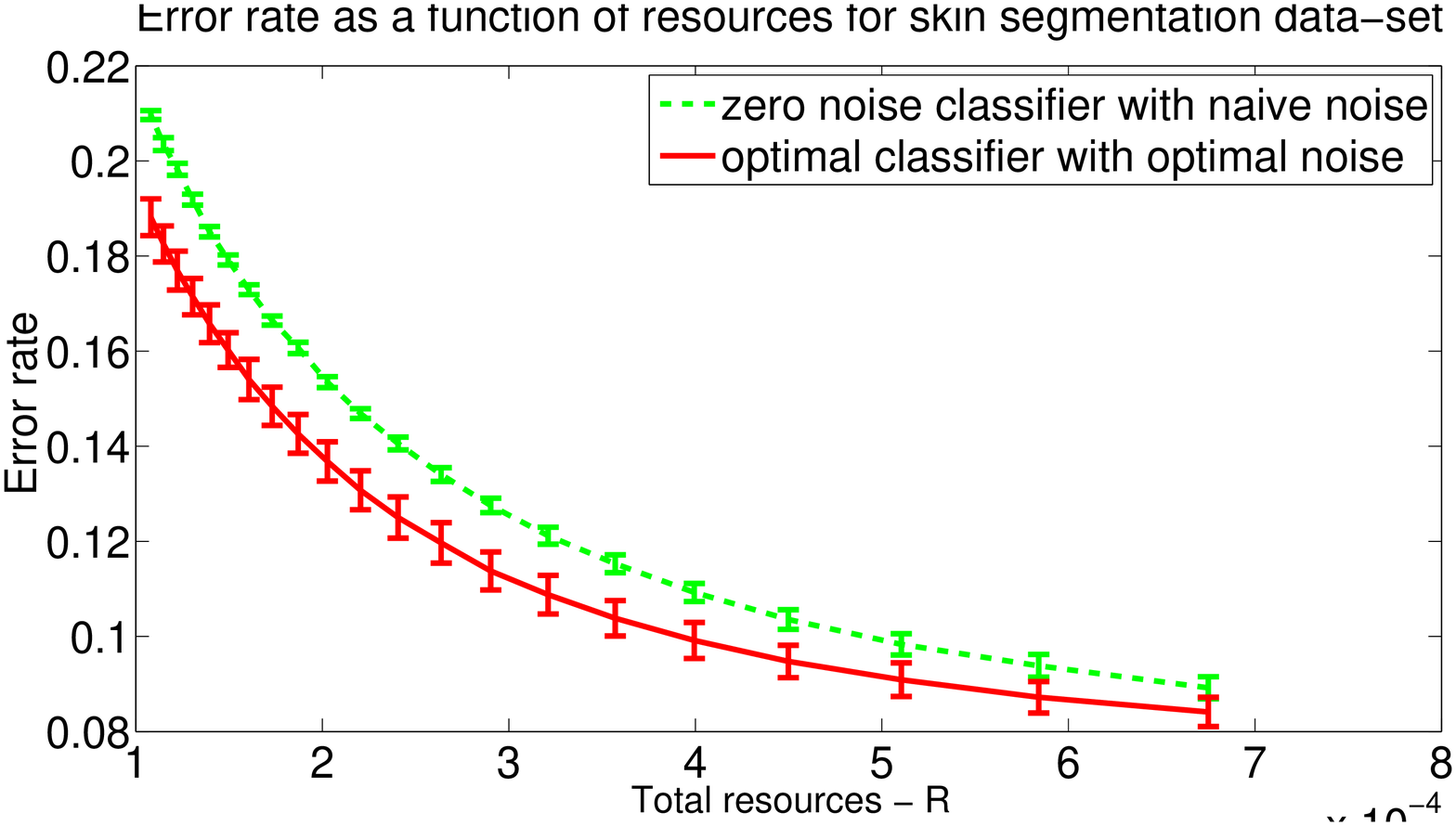}
      \caption{Error rate for skin segmentation data-set }
      \label{error_for_skin}
      \end{figure}
     \begin{figure}
    \centering
   \captionsetup{justification=centering,margin=0.5cm}
       \includegraphics[angle=0,
       width=0.5\linewidth,height=2.5in]{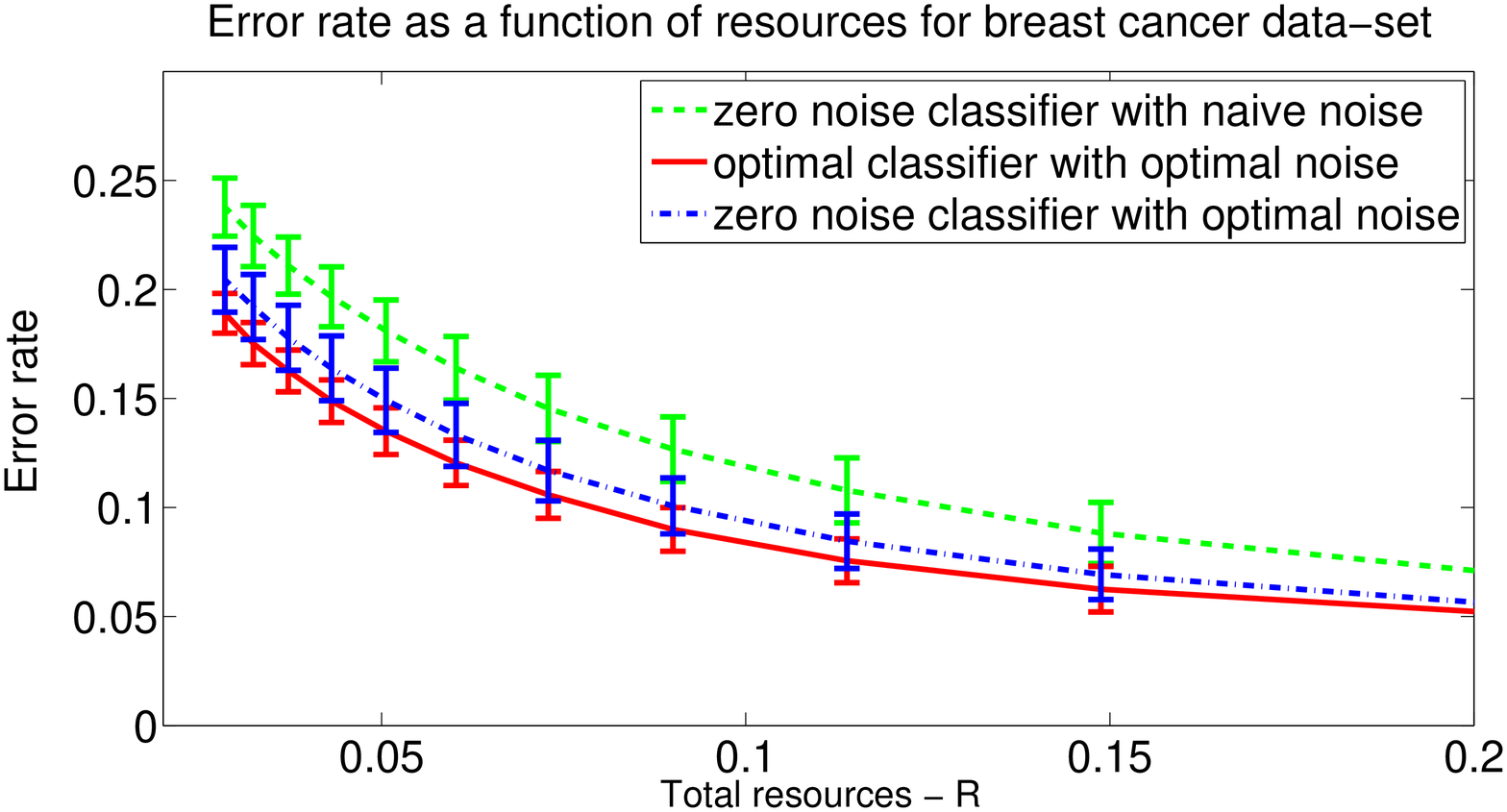}
       \caption{Error rate for breast cancer data-set}
       \label{error_for_breast}
   
   \end{figure}

%\begin{figure}\label{error_for_skin}
%\centering
%\includegraphics[angle=0,
%width=0.8\textwidth]{error-skin}
%\caption{Error rate for skin segmentation data-set}
%\end{figure}

We tested the method on the breast cancer data set from the UCI repository\cite{mangasarian1995breast}.  This data-set contains 9 features that represent measurements from a biopsy and classified each sample as malignant or benign.
The 683 samples were randomly divided,  2/3 of the data was used for training and the remaining 1/3 for testing. The results are depicted in Figure \ref{error_for_breast}. The optimal classifier is different than the zero-noise classifier. In order to demonstrate what portion of the benefit arise from the resource allocation and what portion from the difference between the classifiers we added a plot of the error-rate of the zero-noise classifier when resources are allocated optimally. Most of the benefit comes from the correct allocation of resources. % Interestingly, the optimization problem for the ``revised" classifier is equivalent to adding lasso-regularization which is known to produce classifiers that are more robust to noise. 

%[SM: NOT CLEAR WHAT PROCESS WAS USED FOR BREAST CANCER DATA: DID YOU ADD NOISE TO EACH FEATURE?]

%\begin{figure}\label{error_for_breast}
%\centering
%\includegraphics[angle=0,
%width=0.8\textwidth]{error-breast}
%\caption{Error rate for breast cancer data-set.}
%\end{figure}

\section{Conclusion} \label{section:conclusion}

We presented a method for optimal resource allocation in classification problems along with an analysis of the expected benefits from using this method. Our framework is general and we specialized it for the important special case of linear classifiers with Gaussian noise or with certain adversarial disturbances.

The framework we presented opens up several directions for future research. First, a natural extension of our work is to consider non-linear classifiers. This can be easily done using the ``kernel-trick" computationally. However, while the disturbance (stochastic or adversarial) has a comfortable shape in the input space, this does not necessarily happen in the feature space. This can probably be accommodated using the same techniques as \cite{xu2009robustness} to obtain performance bounds.

Second, an expansion of the framework presented is the case where resources can be further divided between samples such that ``hard'' to classify examples will receive more resources. The key observation for this is the fact that allocation of resources between features is local in nature. The global cost function $L(h,r)$ can be replaced by $l(x,y,h,r)$ and therefore allows deciding on the allocation of resources for each sample separately. The optimal allocation creates a function $l(x,y,h,R)$ that can be used in the method presented in \cite{richman15} to produce optimal allocation between samples.

Finally, the simulation results in this paper include only noise that was artificially generated. This is due to the complexity of creating a closed-loop system that controls the acquisition process. We believe that closing a complete feedback loop in applications such as sensor networks and radar will provide similar benefit to that presented as long as the noise is appropriately modelled.
%other topics we believe need further research:
%\begin{itemize}
%\item  On-line learning was analysed for square-loss. It is interesting to explore other cost-functions. 
%\item  Simulation results include only noise which is artificially generated. This is due to the complexity of creating a closed-loop system which controls the acquisition process. We believe results on real-life noise will be similar to those presented as long as the noise will be modelled appropriately.
%\end{itemize}

\section{Appendix}

\paragraph{Proof of Theorem \ref{optimal-noise} }
		\begin{proof}
		We start by proving the following lemma:
		
			\begin{lemma}\label{convex-lemma}
				%[SM UNDER WHAT ASSUMPTIONS ON THE NOISE - STATE CLEARLY!]
					Let $L(w,b,\sigma)$ be an acceptable loss function.
							If $L(w,b,r)$ is twice differentiable in $r$ then it is convex in $r$.
			\end{lemma}
				
		\begin{proof}
			Since $L$ is twice differentiable in $r$ we can calculate the Hessian
			\begin{equation*}
			\frac{\partial G}{\partial r_i \partial r_j}= \frac{\partial^2L}{\partial^2\sigma}\frac{\partial \sigma}{\partial r_i}\frac{\partial \sigma}{\partial r_j}+\frac{\partial L}{\partial \sigma}\frac{\partial^2 \sigma}{\partial r_i\partial r_j}.
			\end{equation*}
			
			The first term is a positive semi-definite matrix that is multiplied by a positive factor (since $L$ is convex). The second term is the Hessian of $\sigma$ which is  positive semi-definite (since $\sigma(r)$ is convex) multiplied by a positive factor (since $L$ is increasing). Therefore, the Hessian is positive semi-definite and $L$ is convex in $r$. 
			
			\end{proof}

				We now continue to prove the theorem by noting that problem (\ref{problem:stat}) can be rewritten as
				\begin{equation}
							\begin{array}{l}
							\min_{w,b} (\min_r (L(w,b,r)))\\
							s.t \\
							\sum\limits_{i=1}^{d} r_i = R
							\\ r_i>0
							\end{array}
					\end{equation}
					The inner optimization is convex, therefore necessary and sufficient conditions are given by Karush-Khun-Tucker
				\begin{equation*}
								\begin{array}{l}
								\frac{\partial L}{\partial r_i}=\frac{\partial L}{\partial \sigma} \frac{\partial \sigma}{\partial r_i}=-\lambda+\mu_i \\
								\mu_i r_i=0\\
								\sum\limits_{i=1}^{d} r_i = R\\
								 \mu_i \geq 0.
								\end{array}
				\end{equation*}	
				
				 Since $\frac{\partial L}{\partial \sigma} (r)$ is positive and the same for each $r_i$ we can denote $\tilde{\lambda}=\lambda \big(\frac{\partial L}{\partial \sigma} \big)^{-1}$ and obtain the result. 
				\end{proof}

\paragraph{Proof of Theorem \ref{th:R_improve}}
\begin{proof} 
From the definition of $R_{unif}(w,l)$ and  $R_{opt}(w,l)$ it holds that $L(w,b,r_{unif})=L(w,b,r_{opt}(w))$.
Now,
\begin{equation*}
\frac{d|w|_2^2}{R_{unif}}+\mathbb{E}((wx-y)^2)=\frac{|w|_1^2}{R_{opt}}+\mathbb{E}((wx-y)^2)
\end{equation*}

Since the second term is the same in both sides of the equality it is easily derived that
\begin{equation*}
 \frac{R_{unif}(w,l)}{R_{opt}(w,l)}=\frac{d|w|_2^2}{|w|_1^2}.
\end{equation*}

\end{proof}

\paragraph{Proof of Theorem \ref{theorem_adversrial} }
\begin{proof}
Using Lemma \ref{xu_lemma}, problem (\ref{problem1})  turns into:

\begin{equation} \label{problem1_2}
\begin{array}{l l}
\min\limits_{\mathcal{N}_0}\min\limits_{w,b,\xi}
\sup\limits_{\delta \in \mathcal{N}_0 } (w^T\delta)+\sum\limits_{i=1}^{M} \xi_i &\\
s.t:& \\ 
\xi_i \geq 1-y_i(w^Tx_i+b) & i=1\ldots,M \\
\xi_i \geq 0 & i=1\ldots,M \\
\end{array}
\end{equation}
substituting the order of the $\min$ prove the theorem. 
\end{proof}

%\paragraph{Proof of theorem \ref{theorem_algo} }
%\begin{proof}
%Consider,
%\begin{equation}\label{step1}
%\hat{w}_{n+1}=\arg\min\limits_{w,b} L(w,b,r_n,X,Y)
%\end{equation}
%\begin{equation}\label{step2}
%r_{n+1}=\hat{r}(\hat{w}_{n+1})
%\end{equation}
%Now,
%\begin{equation} \label{downgrading_L_proof}
%L_{n+1}=L(w_{n+1},b,r_{n+1},X,Y) \leq L(w_{n+1},b,r_n,X,Y)  \leq  L(w_n,b,r_n,X,Y)=L_n
%\end{equation}
%and the theorem is proven. The first inequality derive form (\ref{step2}) and the second from (\ref{step1}).
%\end{proof}

\paragraph{Proof of Theorem \ref{th:unknown}}

\begin{proof}
We will first cite a slight adaptation of theorem 1 from \cite{zinkevich2003online} (similar adaptation was made in \cite{cesa2011online})
\begin{lemma}
Assume $\max_{t=1,\ldots,T}\mathbb{E}(||\nabla l(w^t,r^t)||^2)\leq ||\nabla l||^2$  then the regret of Algorithm \ref{alg:unknown} satisfies $\mathbb{E}(\sum\limits_{i=1}^{T}(w_tx-y)^2-\min_{|w|_1<B_W} (\sum\limits_{i=1}^{T}(<w_t,x>-y)^2)\leq \frac{B\sqrt{(T)}}{2}+(\sqrt{(T)}-\frac{1}{2})||\nabla l||^2$ where $B=2\sqrt{R^2+B_W^2}$
\end{lemma}  
\end{proof}

Now it is only left to prove that $\max_{t=1,\ldots,T}\mathbb{E}(||\nabla l(w^t,r^t)||^2)\leq 2B_w^2B_{\tilde{x}}^4+2B_{\tilde{x}}^2+\frac{2B_{\tilde{x}}^4B_W^4}{\epsilon^2}$

\begin{IEEEeqnarray*}{rCl}
\mathbb{E}(||\nabla l(w^t,r^t)||^2)&=&\mathbb{E}\big[||(<w,\tilde{x}>-y)\tilde{x}||^2+\sum\limits_{i=1}^{d}w_i^4\frac{{({(\hat{x}_i^{t,2})}^2-{(\hat{x}_i^{t,1})}^2)}^2}{\epsilon^2}]\\
&\leq& 2||w||^2 \mathbb{E}(||\tilde{x}||^4)+2\mathbb{E}(y^2||\tilde{x}||^2)+\frac{||w||^4}{\epsilon^2}\max_{i=1,\ldots,d}\mathbb{E}[(\hat{x}_i^{t,2})^2-{(\hat{x}_i^{t,1})}^2\big]\\
&\leq&2B_w^2B_{\tilde{x}}^4+2B_{\tilde{x}}^2+\frac{2B_{\tilde{x}}^4B_W^4}{\epsilon^2}.
\end{IEEEeqnarray*}

The first inequality results from the fact that $||a+b||^2\leq 2||a|^2+2||b|^2$. The second inequality stands since
 \begin{IEEEeqnarray*}{rCl}
 \mathbb{E}(\tilde{x}^4) &\leq& B_x^4+6B_x^2B_\delta^2+B_\delta^4\\
 \mathbb{E}(\tilde{x}^2) &\leq& B_x^2+B_\delta^2\\
 y^2&=&1.
 \end{IEEEeqnarray*}

\orans

\paragraph{Proof of Theorem \ref{th:uniform}}

\begin{proof}
Denote the noisy measurement $\tilde{x}$ as $x+N$ where $N$ is the noise vector. 
The following relations are obtained by assigning $r_i=\frac{R}{d}$ and $\mathbb{E}||N_i||_2^2=\frac{1}{r_i}=\frac{d}{R}$ :
\begin{equation*}
||\Sigma_tw_t||^2=\sum\limits_{i=1}^{d}\frac{w^2_i}{r_i^2}\leq\frac{d^2}{R^2}B_W^2
\end{equation*}
%\begin{equation*}
%\mathbb{E}(<w_t,N>w_t\Sigma_tN)=\sum\limits_{i=1}^{d}\frac{w^2_i}{r_i^2}\leq\frac{d^2}{R^2}B_W^2
%\end{equation*}
\begin{equation*}
\mathbb{E}(||<W_t,N>||_2^2)=\sum\limits_{i=1}^{d}\frac{w^2_i}{r_i}\leq\frac{d}{R}B_W^2
\end{equation*}
\begin{equation*}
\mathbb{E}(||N||_2^2)=\sum\limits_{i=1}^{d}\frac{1}{r_i}\leq\frac{d^2}{R}.
\end{equation*}

Also,

\begin{equation*}
\mathbb{E}(||<w,N>N||^2)=(\sum\limits_{i=1}^{d}w_iN_i)^2||N||_2^2=\sum\limits_{i,j,k}^{}w_iw_jN_iN_jN_k^2.
\end{equation*}

Since $N_i$ is a zero mean gaussian random variable where for $i \neq j$ $N_i$ and $N_j$ are independent all expectation of odd power in $N_i$ is $0$. in addition $\mathbb{E}(N_i^4)=3\mathbb{E}^2(N_i^2)$. Now.
\begin{IEEEeqnarray*}{rcl}
\mathbb{E}(||<w,N>N||^2)&=&\sum\limits_{i=1}^{d}w_i^2E(N_i^2\sum\limits_{k=1}^{d}N_k^2)\\
&=&\sum\limits_{i=1}^{d}w_i^2E(N_i^4)+\sum\limits_{i,j,i\neq j}^{d}w_i^2E(N_i^2N_j^2))\\
&\leq&3\frac{d^2}{R^2}B_W^2+\frac{d^3}{R^2}B_W^2.
\end{IEEEeqnarray*}  

Now, using the fact that $||a+b||^2<2||a||^2+2||b||^2$ at each stage,

\begin{IEEEeqnarray*}{rCl}
\mathbb{E}(||\nabla_t||_2^2)&=&E_t||2(<w_t,\tilde{x}_t>-y_t)\tilde{x}_t-\Sigma_tw_t||_2^2\\
&\leq&8\mathbb{E}(||(<w_t,\tilde{x}_t>-y_t)\tilde{x}_t||^2)+2||\Sigma_tw_t||^2\\
&\leq&16\mathbb{E}(||(<w_t,x_t>+<w_t,N>))(x_t+N)||^2)+16E(||y_t(x_t+N)||^2)+2||\Sigma_tw_t||^2\\
&\leq&32\mathbb{E}(||(<w_t,x_t>))x_t||^2)+32\mathbb{E}(||(<w_t,N>))N||^2)\\
&&+32\mathbb{E}(||(<w_t,x_t>))N||^2)+32\mathbb{E}(||(<w_t,N>))x_t||^2)\\
&&+16\mathbb{E}(||y_tx_t||^2)+16E(||y_tN||^2)+2||\Sigma_tw_t||^2\\
&\leq&32B_w^2\frac{d^3}{R^2}+98B_w^2\frac{d^2}{R^2}+32B_W^2\frac{d^2}{R}+32B_W^2\frac{d}{R}+16\frac{d^2}{R}+32B_W^2B_x^4+16=G 
\end{IEEEeqnarray*}

where the last inequality is due to the relations above.

\end{proof}

\paragraph{Proof of Theorem \ref{th:efficent}}

\begin{proof}
Now, $r_i(w)=\frac{R}{2d}+\frac{Rw_i}{2|w|_1}$. Therefore $r_i\geq \frac{R}{2d}$ and  $r_i\geq \frac{Rw_i}{2|w|_1}$.
This results in  $\mathbb{E}||N_i||_2^2\leq \frac{2d}{R}$ and  $\mathbb{E}||N_i||_2^2\leq \frac{2|w|_1}{Rw_i}$ 

Now,
\begin{equation*}
\begin{array}{l}
||\Sigma_tw_t||^2=\sum\limits_{i=1}^{d}\frac{w^2_i}{r_i^2}\leq\frac{4d}{R^2}B_W^2\\
%\mathbb{E}(<w_t,N>w_t\Sigma_tN)=\sum\limits_{i=1}^{d}\frac{w^2_i}{r_i^2}\leq\frac{4d}{R^2}B_W^2\\
\mathbb{E}(||<W_t,N>||_2^2)=\sum\limits_{i=1}^{d}\frac{w^2_i}{r_i}\leq\frac{2}{R}B_W^2\\
\mathbb{E}(||N||_2^2)=\sum\limits_{i=1}^{d}\frac{1}{r_i}\leq\frac{2d^2}{R}
\end{array}
\end{equation*}

In a similar fashion to the derivation in the proof of Theorem \ref{th:uniform} it results that
%\begin{IEEEeqnarray*}{rcl}
\begin{equation*}
E(||<w,N>N||^2)=\sum\limits_{i=1}^{d}w_i^2E(N_i^4)+\sum\limits_{i,j,i\neq j}^{d}w_i^2E(N_i^2N_j^2))\leq 12\frac{d}{R^2}B_W^2+4\frac{d^2}{R^2}B_W^2.
\end{equation*}

%E(||<w,N>N||^2)&=&\sum\limits_{i=1}^{d}w_i^2E(N_i^4)+\sum\limits_{i,j,i\neq j}^{d}w_i^2E(N_i^2N_j^2))\\
%&\leq&2\frac{d^2}{R^2}B_W^2+12\frac{d}{R^2}B_W^2.
%\end{IEEEeqnarray*}  

The rest of the proof is similar to that of Theorem \ref{th:uniform} using the above relations instead of the corresponding relations in Theorem \ref{th:uniform}. 

\end{proof}

\orane
%\clearpage
\footnotesize{
%\singlespacing
\onehalfspacing
\bibliography{mybib2}{}
\bibliographystyle{abbrv}
}

\clearpage

\end{document}